\documentclass[1p]{elsarticle}

\usepackage{hyperref}
\usepackage{amsmath}
\usepackage{url}
\usepackage{cancel}
\usepackage{latexsym}
\usepackage{footnote}
\usepackage{amsmath}
\usepackage{enumitem}
\usepackage{amsthm}
\usepackage{amsfonts}
\usepackage{amssymb}
\usepackage[lined,boxed,commentsnumbered]{algorithm2e}
\usepackage{graphicx}
\graphicspath{{/home/admin123/BigDataIEEE/IEEEtran/plots},{/home/admin123/Big_Data_Paper_Code_Data/Updates_Paper_TKDE/plots/}}
\DeclareGraphicsExtensions{.pdf,.png,.jpg}
\usepackage{float}
\usepackage[font={small,it}]{caption}
\usepackage{csquotes}
\usepackage{multicol}
\usepackage{wrapfig}
\usepackage[bottom,flushmargin,hang,multiple]{footmisc}
\usepackage{etoolbox}
\makesavenoteenv{tabular}

\journal{Journal of \LaTeX\ Templates}

\begin{document}

\begin{frontmatter}

\title{Fast Gaussian Process Regression for Big Data}

\author{Sourish Das\fnref{cmiauthor1}}
\author{Sasanka Roy\fnref{isiauthor}}
\author{Rajiv Sambasivan \fnref{cmiauthor1}\corref{mycorrespondingauthor}}

\fntext[cmiauthor1]{Chennai Mathematical Institute}
\fntext[isiauthor]{Indian Statistical Institute}

\cortext[mycorrespondingauthor]{Corresponding author}
\ead{rsambasivan@cmi.ac.in}

\begin{abstract}
Gaussian Processes are widely used for regression tasks. A known limitation in the application of Gaussian Processes to regression tasks is that the computation of the solution requires performing a matrix inversion. The solution also requires the storage of a large matrix in memory. These factors restrict the application of Gaussian Process regression to small and moderate size data sets. We present an algorithm that combines estimates from models developed using subsets of the data obtained in a manner similar to the bootstrap. The sample size is a critical parameter for this algorithm. Guidelines for reasonable choices of algorithm parameters, based on detailed experimental study, are provided. Various techniques have been proposed to scale Gaussian Processes to large scale regression tasks. The most appropriate choice depends on the problem context. The proposed method is most appropriate for problems where an additive model works well and the response depends on a small number of features. The minimax rate of convergence for such problems is attractive and we can build effective models with a small subset of the data. The Stochastic Variational Gaussian Process and the Sparse Gaussian Process are also appropriate choices for such problems. These methods pick a subset of data based on theoretical considerations. The proposed algorithm uses bagging and random sampling. Results from experiments conducted as part of this study indicate that the algorithm presented in this work can be as effective as these methods. 
\end{abstract}


\end{frontmatter}

\section{Introduction}\label{sec:introduction}
Gaussian Processes (GP) are attractive tools to perform supervised learning tasks on complex datasets on which traditional parametric methods may not be effective. They are also easier to use in comparison to alternatives like neural networks (\cite{Rasmussen06gaussianprocesses}). Gaussian Processes offer some practical advantages over Support Vector Machines (SVM) (\cite{ghahramani_mlss_2011}). They offer uncertainty estimates with predictions. The kernel and regularization parameters can be learned directly from the data instead of using cross validation. Feature selection can be incorporated into the learning algorithm. For regression, exact inference is possible with Gaussian Processes. To apply Gaussian Processes to classification, we need to resort to approximate inference techniques such as Markov Chain Monte Carlo, Laplace Approximation or Variational Inference. Even though exact inference is possible for Gaussian Process regression, the computation of the solution requires matrix inversion. For a dataset of size $n$, the time complexity of matrix inversion is $O(n^3)$. The space complexity associated with storing a matrix of size $n$ is $O(n^2)$. This restricts the applicability of the technique to small or moderate sized datasets.

In this paper we present an algorithm that uses subset selection and ideas borrowed from bootstrap aggregation to mitigate the problem discussed above. Parallel implementation of this algorithm is also possible and can further improve performance.

The rest of this paper is organized as follows: In section \ref{pf},  we discuss the problem context. In section \ref{ps},  we present our solution to the problem. Our solution is based on combining estimators developed on subsets of the data. The selection of the subsets is based on simple random sampling with replacement (similar to what is done in the bootstrap). The size of the subset selected is a key aspect of this algorithm. This is determined empirically. We present two methods to determine the subset size. We present the motivating ideas leading to the final form of the algorithm.  When the model is an additive structure of univariate components, this has attractive implications on the convergence rate (\cite{stone1985additive}). An additive model worked well for the datasets used in this study. Relevant facts from Minimax theory for non-parametric regression,  that are consistent with the experimental results reported in this work, are presented. In section \ref{pw},  we present a brief summary of related work. Applying Gaussian Processes to large datasets has attracted a lot of interest from the machine learning research community. Connecting ideas to research related to the algorithm reported in this work are presented.\\ 
Selecting parameters for an algorithm is an arduous task. However this algorithm has only two important parameters, the subset size and the number of estimators. We provide guidelines to pick these parameters based on detailed experiments across a range of datasets. In section \ref{sec:effect_of_parameters} we provide experimental results that provide insights into the effect of the parameters associated with the algorithm. In section \ref{ards},  we illustrate the application of our algorithm to synthetic and real world data sets. We applied the algorithm developed in this work to data sets with over a million instances. We compare the performance of the proposed method to the Sparse Gaussian Process (\cite{titsias2009variational}) and the Stochastic Variational Gaussian Process (\cite{hensman2013gaussian}). The inputs required by these algorithms are similar to the inputs required for the proposed method and therefore are applicable in a similar context.
We compare the estimates obtained from the reported algorithm with two other popular methods to perform regression on large datasets - Gradient Boosted Trees (using XGBoost, \cite{chen2016xgboost}) and the Generalized Additive Model (GAM)(\cite{friedman2001elements}).Results from experiments performed as part of this study show that accuracies from the proposed method are comparable to those obtained from Gradient Boosted Trees or GAM's. However there are are some distinct advantages to using a Gaussian Process model. A Gaussian Process model yields uncertainty estimates directly whereas methods like Gradient Boosted Trees do not provide this (at least directly). A Gaussian Process model is also directly interpretable in comparison to methods like Gradient Boosted Trees or Neural Networks. Therefore, the proposed method can yield both explanatory and predictive models. It is possible to use stacking (\cite{wolpert1992stacked}) to combine the estimates from the proposed model with those obtained from a competing model (like Gradient Boosted Trees) and obtain higher accuracies. Combining a Gaussian Process solution with XGBoost has been used by \cite{lloyd2014gefcom2012}.\\
In section \ref{dr},  we present the conclusions from this work. The contribution of this work is as follows. The proposed method to perform Gaussian Process regression on large datasets has a very simple implementation in comparison to other alternatives,   with similar levels of accuracy. The algorithm has two key parameters - the subset size and the number of estimators. Detailed guidelines to pick these parameters are provided. The choice of a method to scale Gaussian Process regression to large datasets depends on the characteristics of the problem. This is discussed in section \ref{pw}. The proposed method is most effective for problems where the response depends on a small number of features and the kernel characteristics are unknown. In such cases, exploratory data analysis can be used to arrive at appropriate kernel choices \cite{duvenaud2014automatic}. Additive models may work well for these problems. Appropriate preprocessing like principal component analysis can be used if needed to create additive models. The rate of convergence for additive models is attractive (\cite{stone1982optimal}). This implies that we can build very effective models with a small proportion of samples. Sparse Gaussian Processes see \cite{snelson2005sparse} and Stochastic Variational Gaussian Processes \cite{hensman2013gaussian} are also appropriate for such problems. These require a more complex implementation and may require extensive effort to tune the optimization component of the algorithm (see \cite{hensman2013gaussian}). Results of the experiments conducted as part of this study show that the proposed method can match or exceed the performance of these methods.
\hfill 
\section{Problem Formulation} \label{pf}
A Gaussian Process $y$ with additive noise can be represented as:
\begin{equation}\label{eqn:gp_model}
 y = f(x) + \eta.
\end{equation}
Here :
\begin{itemize}
\item $y$ represents the observed response.
\item $x$ represents an input vector of covariates.
\item $\eta$ represents the noise. The noise terms are assumed to be identical, independently distributed (IID) random variables drawn from a normal distribution with variance $\sigma_{n}^2$.

\item $f$ represents the function being modeled. It is a multivariate normal with mean function $\mu(x)$ and covariance function  $\mathcal{K}(x)$.
\end{itemize}
 
If we want to make a prediction for the value of the response at a test point $X^*$, then the predictive equations are given by (see \cite{Rasmussen2005}):
\small
\begin{equation}
f_{*} | X, y, X_{*}  \sim \mathcal{N} \left(\overline{f}_{*}, cov(f_{*}) \right), \label{eqn:1}\\
\end{equation}
\begin{align}
\begin{split}
\overline{f}_{*}  & = \mathbb{E}[f_{*}| X, y, X_{*}], \\
 & = \mathcal{K}(X_{*},X)[(\mathcal{K}(X,X) + \sigma_{n}^2 \textbf{I}] ^{-1} y,  \label{eqn:2}
 \end{split}
\end{align}

\begin{equation}
\mathbb{V}(f_{*})  =   \mathcal{K}(X_{*}, X_{*}) - \\
{\mathcal{K}(X_{*}, X)[\mathcal{K}(X,X) + \sigma_{n}^2 \textbf{I}]^{-1} \mathcal{K}(X,X_*)}.  \label{eqn:gp_var_est}
\end{equation}
\normalsize
Here:
\begin{itemize}
\item $f_{*}$ is the value of the function at the test point $X_{*}$.
\item $\mathcal{K}(X_*,X)$ represents the covariance between the test point and the training set.
\item $\mathcal{K}(X,X)$ represents the covariance matrix for the training set.
\item \textbf{I} is the identity matrix.
\end{itemize}
Equation \eqref{eqn:2} is the key equation to make predictions. An inspection of equation \eqref{eqn:2} shows that this equation requires the computation of an inverse. For a training dataset of size $n$, computation of the inverse has $\mathcal{O}(n^3)$ time complexity. This is one of the bottle necks in the application of Gaussian Processes. Calculation of the solution also requires the storage of the matrix $(\mathcal{K}(X,X) + \sigma_{n}^2 \textbf{I})$. This is associated with a space complexity of $\mathcal{O}(n^2)$. This is the other bottle neck associated with Gaussian Processes. The uncertainty associated with our prediction is provided by Equation \ref{eqn:gp_var_est}. The covariance $\mathcal{K}(X,X)$ is expressed in terms of a kernel that is appropriate for the modeling task. The kernel is associated with a set of hyper-parameters. These may be known for example if the modeling task has been well studied or unknown if the problem has not been well studied. In this work we treat these kernel parameters as unknown. When the kernel parameters are unknown, these are estimated using maximum likelihood estimation. The marginal log likelihood is given by (see \cite{Rasmussen2005}):

\begin{equation}\label{eqn:ml}
\log\big(y|X\big) = \frac{1}{2}y^T\big(\mathcal{K} + \sigma_{n}^2.\mathbf{I}\big)^{-1}.y  -\frac{1}{2} \log\big|\mathcal{K} + \sigma_{n}^2.\mathbf{I}\big| - \frac{n}{2} \log\big(2\pi\big).
\end{equation}

Using an appropriate optimization technique with Equation \eqref{eqn:ml} as the objective function, the hyper-parameters of the kernel ($\mathcal{K}$) can be determined.   

\section{Proposed Solution} \label{ps}
Since GP regression is both effective and versatile on complex datasets in comparison to parametric methods, a solution to the bottlenecks mentioned above will enable us to apply GP regression to large datasets that are complex. We run into such datasets routinely in this age of big data. Our solution to applying Gaussian Process regression to large data sets is based on developing estimators on smaller subsets of the data. The size of the subset and the desired accuracy are key parameters for the proposed solution. The accuracy is specified in terms of an acceptable error threshold, $\epsilon$. We pick $K$ smaller subsets of size $N_s$ from the original data set (of size $N$) in a manner similar to bootstrap aggregation. We develop a Gaussian Process estimator on each subset. The GP fit procedure includes hyper-parameter selection using the likelihood as the objective function. We want a subset size such that when we combine the estimators developed on the $K$ subsets, we have an estimator that yields a prediction error that is acceptable. The rationale for this approach is based on results from Minimax theory for non-parametric regression that are presented later in this section. To combine estimators, we simply average the prediction from each of the $K$ estimators.

The time complexity for fitting a Gaussian Process regression on the entire dataset of size $N$ is $O(N^3)$. The algorithm presented above requires the fitting of $K$ Gaussian Process regression models to smaller size datasets ($N_s$). Therefore, the time complexity is $O(K.N_{s}^{3})$.

We present two methods to determine the subset size:
\begin{enumerate}
\item Estimating the subset size using statistical inference. For a dataset of size $N$, the subset size is expressed as:
\begin{equation}
N_{s} = N^\delta, where \quad (0 < \delta < 1).
\end{equation}\label{eqn:exp_est_1}
\noindent $\delta$ is a random variable and is determined based on inference of a proportion using a small sample. The details of this method are presented in the next subsection.
\item Estimating the subset size using an empirical estimator. We use the following observations to derive this empirical estimator:
\newtheorem{lemma}{Lemma}
\begin{enumerate}
\item The subset size $N_s$, should be proportional to the size of the dataset, $N$:
\begin{equation*}
N_s \propto N.
\end{equation*}
We posited that as the size of the dataset increases, there is possibly more detail or variations to account for in the model. Therefore larger datasets would require larger sample sizes. 
\item  The sample size $N_s$, should be a decreasing function of the desired error rate ($\epsilon$). This means that decreasing the desired error rate should increase the sample size.
\begin{equation*}
N_s \propto \frac{1}{g(\epsilon)}, 
\end{equation*}

\end{enumerate}
\noindent where $g(\epsilon)$ is an increasing function. Application of the above observations yields the following estimator for sample size:
\begin{equation} \label{eqn:sasanka_estimator}
N_{s} = \frac{N^{\delta(N)}}{g(\epsilon)}.
\end{equation}
Here:
\begin{description}
\item[$\delta(N)$] is a function that characterizes the fact that an increase in $N$ should increase the sample size $N_s$.
\item[$g(\epsilon)$] is a function that characterizes the fact that sample size should increase as $\epsilon$ (desired error level) decreases.
\end{description}
\end{enumerate}
The algorithm is summarized in Algorithm ~\ref{algo_resamp}

\IncMargin{1em}
\begin{algorithm}[ht]
\SetKwData{Sample}{$N_s$}\SetKwData{fsample}{$\hat{f}_i$}\SetKwData{fom}{$f_{resampled}(x)$}\SetKwData{Up}{up}
\SetKwFunction{SampleWithReplacement}{SampleWithReplacement}\SetKwFunction{FitGP}{FitGP}
\SetKwInOut{Input}{input}\SetKwInOut{Output}{output}

\Input{A dataset $\mathcal{D}$ of size N, $\delta$, $K$}

\Output{An estimator f that combines the estimators fitted from resampling}
\BlankLine

\For{$i\leftarrow 1$ \KwTo $K$}{
\tcc{select a sample from $\mathcal{D}$. Two ways of selecting the sample size are presented}
\Sample$\leftarrow$\SampleWithReplacement{$\mathcal{D}, \delta$}\;
\tcc{A kernel is fit for each sample. Hyper-parameter selection is done for each sample. This computation can be parallelized.}
\fsample$\leftarrow$\FitGP{$N_s$}\;

} 
\tcc{the estimate for a point $x \in \mathcal{D}_{test}$ (the test dataset) is the average of the estimates from the K estimators fitted above.}
\fom $\leftarrow \frac{1}{K}\sum_{i=1}^{i = K} \hat{f}_i(x)$ 
\caption{Gaussian Process Regression Using Resampling}\label{algo_resamp}
\end{algorithm}\DecMargin{1em}

The proposed algorithm is based on model averaging as described in \cite[Chapter 14]{bishop2006pattern}. This is similar to combining estimates from regression trees in the Random Forest(\cite{breiman2001random}) algorithm. Given $i = \{1,2,\hdots, K\}$ models, the conditional distribution of the response at a point $X_j$ is obtained from:
\begin{equation}
p(Y|X_j) = \sum_{i =1}^{i = K} p(Y|i, X_j).p(i).
\end{equation}
If each of the models is equally probable, then $p(i) = \left(\frac{1}{K}\right)$. If each of $p(Y|i,X_j)$ is a Gaussian $\mathcal{N}(\mu_i, \Sigma_i)$, as would be the case when each of the estimators is a GP based on Equation \ref{eqn:gp_model}, then we have the following:
\begin{equation}\label{eqn:model_averaging}
p(Y|X_j) \thicksim \mathcal{N}( \mu_C(X_j),\quad \sigma ^2_{C}(X_j)),
\end{equation}
Where:
\begin{description}
\item $\mu_C(X_j) = \frac{1}{K}. \sum_{i=1}^{i=K} \mu_i(X_j), $
\item $\sigma^2_C(X_j) = \frac{1}{K^2} \sum_{i=1}^{i=K} \sigma ^2 _{i}(X_j).$
\end{description}
Model combination is an alternative approach to combining estimates from a set of models. The product of experts model is an example of this approach. In the product of experts approach, each model is considered to be an independent expert. The conditional distribution of the response at a point $X_j$ is obtained from:

\begin{equation}\label{eqn:poe_eqns}
p(Y|X_j) = \frac{1}{Z} \prod_{i = 1}^{i=K} p_i(Y|X_j), 
\end{equation}
\noindent where $Z$ is the normalization constant. \cite{cao2014generalized} report a study where the individual experts are Gaussian Processes. In this case, each of $p_i(Y| X_j)$ are $\mathcal{N}(\mu_i(X_j),\  \sigma^2_h(X_j))$. The mean and variance associated with $p(Y|X_j)$ are:
\begin{align} \label{eqn:poe}
\mu_{POE}(X_j) &= \sigma^2_{POE} (X_j).\left(\sum_{i=1}^{i=K}\mu_i(X_j)T_i(X_j)\right), \\ 
T_{POE}(X_j) &= \sum_{i=1}^{i=K} T_i(X_j),\\
\sigma^2_{POE} (X_j) &= \left(T_{POE}(X_j) \right)^{-1}.
\end{align} 
Here $T_i(X_j)$ is the precision of the expert $i$ at point $X_j$. The results from using a product of experts model are also included in this study (see section \ref{ards}). 

Model averaging and model combination (Equation \eqref{eqn:model_averaging} and Equation \eqref{eqn:poe}) are simply ways to combine estimates from the component estimators used in Algorithm \ref{algo_resamp}. They do not specify any information about the character of the estimator developed using Algorithm \ref{algo_resamp}. To do this, we need the following preliminaries:

\newtheorem{assump}{Assumption}

\begin{assump}
\label{assump_1}
The function being modeled is in the Reproducing Kernel Hilbert Space of the kernel ($\sigma$) of the estimator. The reproducing property of the kernel can be expressed as:
\begin{equation*}
f(t) = <f(.), \sigma(., t)>  \quad \forall t \in T.
\end{equation*}
Here $T$ represents the index set to select the predictor instances.

\end{assump}

Algorithm \ref{algo_resamp} makes use of $K$ estimators, $f_1$, $f_2$, $\ldots$, $f_K$. The reproducing kernels associated with the estimators are 
$\sigma_1$, $\sigma_2$, $\ldots$, $\sigma_K$.  We show that the estimator resulting from combining these $K$ estimators is a Gaussian Process. 

\begin{lemma} \label{rkhs_lemma}
The estimator $f_{rm}$ obtained from the individual estimators $f_1$, $f_2$, $\ldots$, $f_K$ using:
\begin{equation*}
f_{rm}(t) = \frac{1}{K}\sum_{i=1}^{i=K} f_i(t), 
\end{equation*}
is associated with the following kernel:
\begin{equation} \label{estimator_kernel}
\sigma_{rm} = \frac{1}{K}\sum_{i=1}^{i=K} \sigma_{i}\quad .
\end{equation}
\end{lemma} 

\begin{proof}
From Assumption ~\ref{assump_1}, the kernel associated with each estimator has the reproducing property. So the following equation holds:
\scriptsize
\begin{equation*}
<f(.), \sigma_1(.,t)> = <f(.), \sigma_2(.,t)> = \ldots = <f(.), \sigma_K(.,t)> = f(t)\ .
\end{equation*}
\normalsize
Consider the kernel $\sigma_{rm} = \frac{1}{K}\sum_{i=1}^{i=K} \sigma_{i}$ and the inner product $<f(.), \sigma_{rm}>$. The inner product can be written as:
\scriptsize
\begin{align*}
<f(.), \sigma_{rm}> &= \frac{1}{K}\overbrace{[<f(.), \sigma_1(.,t)> + \ldots +
<f(.), \sigma_K(.,t)>]}^{K\quad times}\\
&= \frac{1}{K}\overbrace{[f(t) + \ldots +
f(t)]}^{K\quad times}\\
&=\frac{1}{K}\  K\  f(t)\\
&= f(t)\ .
\end{align*}
\normalsize
\end{proof}

Using Lemma \eqref{rkhs_lemma}, the estimator developed using Algorithm \ref{algo_resamp} is a Gaussian Process. This Gaussian Process is associated with a kernel defined by Equation \eqref{estimator_kernel}. This kernel is defined as a mixture of the kernels used by the individual estimators. The uncertainty estimate at any point $X$ can be obtained using Equation \eqref{eqn:gp_var_est}.

Consistency of estimators is concerned with the asymptotic properties of the developed estimator as we increase the sample size. Consistency of Gaussian Processes  has been widely studied (see \cite[Chapter ~7, Section 7.1]{Rasmussen06gaussianprocesses}) for details. The convergence rate for the algorithm is a characteristic of great interest. This characteristic tells us the rate at which we can drop the error as we increase the sample size. Minimax theory provides a framework to assess this (see \cite{tsybakov2009introduction} or \cite{gyorfi2006distribution}) . A famous result by \cite{stone1982optimal} states the minimax rate of convergence for non-parametric regression is $n^{-\frac{\alpha}{(2\alpha + d)}}$, where $\alpha$ describes the smoothness of the function being modeled and $d$ represents the dimensionality of the feature space. This suggests that the rate of convergence is affected by the curse of dimensionality. However as noted by \cite{yang2015minimax}, the following factors usually mitigate the curse of dimensionality:
\begin{enumerate}
\item The data lie in a manifold of low dimensionality even though the dimensionality of the feature space is large. \cite{yang2016bayesian} reports a method for such problems.
\item The function being modeled depends only on a small set of predictors. All datasets reported in this study had this characteristic. Feature selection can be performed in Gaussian Processes using Automatic Relevance Determination  (see \cite[Chapter 5, Section 5.1]{Rasmussen06gaussianprocesses}).
\item The function being modeled admits an additive structure with univariate components. The minimax rate of convergence for this problem is very attractive (see \cite{stone1985additive}). For all datasets reported in this study such an additive model yielded good results.
\end{enumerate}
Feature Selection is therefore a critical first step. This is discussed in section \ref{ards}. The additive structure and the dependence on a very small set of predictors suggest that we can get reasonable models with a small subset of the data. This was consistent with the experimental results reported in this work. Simple preprocessing like Principal Component Analysis (PCA) could be used to reduce the number of relevant features. PCA also makes the features independent. The data for this study came from public repositories and from very diverse application domains. This suggests that datasets with these characteristics are not uncommon. When data lie in a manifold, methods such as \cite{yang2016bayesian} may be more appropriate.

\cite{zaytsev2017minimax} investigate the minimax interpolation error under certain conditions (known covariance, stationary Gaussian Process) . As noted in this recent study, theoretical work in estimating the convergence rate using minimax theory is an active area of research. In this work we investigated empirical estimation of the sample size. We describe two methods to determine the sample (subset) size. 

\subsection{Estimating Subset Size on the Basis of Inference of a Proportion} \label{sec:emp_exp_est} 

Since $\delta$ takes values between 0 and 1, it can be interpreted as a proportion. We treat it as a random variable that can be inferred from a small sample. 
To estimate $\delta$, we do the following: We pick a small sample of the original dataset by simple random sampling. We start with a  small value of $\delta$ and check if the prediction error with this value of $\delta$ is acceptable. If not we increment $\delta$ until we arrive at a $\delta$ that yields an acceptable error. This procedure yields the smallest $\delta$ value that produces an acceptable error on this sample. Since the size of this dataset is small, the above procedure can be performed quickly. This technique yielded reliable estimates of $\delta$ on both synthetic and real world datasets.

\subsection{Empirical Estimation of the Subset Size}\label{sec:emp_est_sasanka}
Equation \ref{eqn:sasanka_estimator} provides the functional form for the second empirical estimator of the subset size $N_s$. Appropriate choices for $\delta(N)$ and $g(\epsilon)$ are based on observations from  the experimental evaluation of the parameters of Algorithm \ref{algo_resamp}. Choices that provided good results are provided in section \ref{sec:effect_dataset_size}.

\section{Related Work} \label{pw}
\cite[Chapter~8]{Rasmussen2005} provides a detailed discussion of the approaches used to apply Gaussian Process regression to large datasets. \cite {quinonero2007approximation} is another detailed review of the various approaches to applying Gaussian Processes to large datasets. The choice of a method appropriate for a regression task is dependent on the problem context. Therefore we discuss the work related to scaling Gaussian Process regression taking the problem context into consideration.\\

If the data associated with the problem has been well studied and kernel methods have been successfully applied to the problem, then we may have reasonable insights into the nature of the kernel appropriate for the regression task. We may be able to arrive at the kernel hyper-parameters quickly from a small set of experiments. On the other hand if the problem and data is new, then we may not have a lot of information about the kernel. In general, scaling Gaussian Process regression to large datasets has two challenges:
\begin{enumerate}
\item Finding a kernel that has good generalization properties for the dataset.
\item Overcoming the computational hurdles - $O(N^3)$ for training and $O(N^2)$ for storage.
\end{enumerate}

Learning a kernel that has good generalization properties is a related area of research in Gaussian Processes (see \cite{wilson2014thesis}, \cite{wilson2013gaussian}. When a good kernel representation has been learned, there are many techniques to overcome the computational hurdles.  The Nystrom method to approximate the Gram matrix (\cite{drineas2005nystrom}) and the Random Kitchen Sinks (\cite{rahimi2008random}) are probably the most well known. The Random Kitchen Sinks approach  maps the data into a low dimensional feature space and learns a linear estimator in this space. It should be noted that these methods work well when the problem needs a stationary kernel for which we know the hyper-parameters. Using "sensible defaults" for hyper-parameters and applying these techniques to problems that require a non-stationary kernel may yield poor results. For example with the airline dataset, described later in this study (see section \ref{sec:datasets}), the Random Kitchen Sinks cannot be used directly and would require suitable preprocessing (like removing simple trends or using a mean function) so that a stationary kernel would be applicable. Using the Random Kitchen Sinks directly with no preprocessing and using default kernel choices provided with the scikit-learn \cite{scikit-learn} implementation yielded poor results (RMSE of 31.49 as opposed to 8.75 with the proposed method).\\ 
Using a kernel learning approach to determine a good kernel representation and then solving the computational hurdles independently is one way to approach scaling Gaussian Process regression to large datasets. Another approach to determine the appropriate kernel is to use exploratory data analysis. Guidelines to pick kernels based on exploratory analysis of the data is provided in \cite{duvenaud2014automatic}. This is a practical approach when the number of relevant features is not too many, as was the case with the datasets used in this study. It should be noted that hyper-parameters for these choices still need to be specified. We may be able to build additive models using this approach. Appropriate preprocessing could help, for example principal component analysis can be applied to make the features independent. Minimax theory for non-parametric regression indicates that the convergence rate for additive models is very attractive. We can build effective models with a small proportion of the data.\\
The choice of kernel hyper-parameters is critical and can affect the performance. When datasets are large and the kernel hyper-parameters are unknown, we need algorithms that can address both these issues. Ideally, the algorithm should be able to work with both stationary and non-stationary kernels. The proposed algorithm is one such candidate.
Sparse Gaussian Processes (\cite{titsias2009variational}) and Stochastic Variational Gaussian Processes (\cite{hensman2013gaussian}) are two others. Like the proposed algorithm, these algorithms require the specification of a input size. A subset of points is selected from the dataset for training. The criteria for subset selection is different in each case. These algorithms do not require the specification of the kernel hyper-parameters. These are estimated from the data. Stochastic Variational Gaussian Processes can require considerable manual tuning of the optimization parameters. Typical implementations (like \cite{gpy2014}) for Sparse GP and Stochastic Variational GP use stochastic gradient descent for hyper-parameter optimization. This explores the entire dataset in batch size increments. \cite{hensman2013gaussian} report the details associated with picking the parameters for the optimization task (learning rates, momentum, batch sizes etc.). In contrast, sample sizes with the proposed algorithm even for datasets with over million instances are typically small (order of few hundred instances). Learning hyper-parameters over small datasets is considerably easier. The experiments reported in this work required no tuning effort. We report the performance of Sparse Gaussian Process, Stochastic Variational Gaussian Process and the proposed method on a variety of datasets in section \ref{ards}\\

The proposed algorithm uses ideas that have proven effectiveness with other machine learning techniques. Bagging has been used to improve performance using regression trees (Random Forests, \cite{breiman2001random}). Like with Random Forests, the algorithm uses model averaging to combine estimates from component Gaussian Process Regressions. Dropout (\cite{srivastava2014dropout}) is a technique used in neural networks to prevent over fitting. Dropping random units achieves regularization in neural networks. In the proposed algorithm, selecting a sample can be viewed as dropping random instances from the training dataset. Sparse Gaussian Processes and Stochastic Variational Gaussian Processes use theoretical ideas to select a small subset of points to develop a Gaussian Process regression model. This study suggests that combined with model averaging, random selection of the subset can also work well.\\

For datasets with a Cartesian product structure, by imposing a factorial design of experiment scheme on the dataset and decomposing the covariance matrix as a Kronecker product \cite{belyaev2016computationally} discuss an approach to scaling Gaussian Process regression. This approach also uses a prior on the hyper-parameters to deal with anisotropy. So the implementation is quite complex. Using a divide an conquer strategy is another theme in scaling Gaussian Process regression to large datasets. The Bayesian Committee  Machine (BCM) (\cite{tresp2000}), is an idea related to the algorithm presented in this work. The BCM proposes a partition of the dataset into $M$ parts. An estimator is developed on each partition. The BCM does not choose a subset of the partition. It uses the entire partition for developing the estimator. This is the key difference between the method proposed in this work and the BCM. The estimates from each estimator are assumed to be independent. The BCM assumes that computing a GP solution on the partitions of the dataset is computationally tractable because the partition sizes are small. Datasets encountered today are much larger than those reported in \cite{tresp2000}. In present day datasets the partitions of the dataset based on guidelines provided in \cite{tresp2000} would be very big and computing a full Gaussian Process solution on them may not be computationally tractable. Even when the partition size is not big enough to create computational hurdles, using all the data may result in over fitting. Using a hierarchical model as in \cite{deisenroth2015distributed} or \cite{park2010hierarchical} are possible ways to work around the size of the partitions, however this requires a complex implementation to partition and recombine computations. 

The Locally Approximate Gaussian Process \cite{gramacy2015local} fits a local Gaussian Process for a prediction point using a local neighborhood and local isotropy assumption. This method too requires some tuning, the size of neighborhood and method to choose the neighbors are important parameters. For datasets where the local isotropy assumption works well and when the size of the test set is small, this method might be useful. When prediction is required at a large number of test points, this method might be slow if we use a large neighborhood. For example with the airline dataset, described in section \ref{sec:datasets}, using the defaults provided with the laGP (\cite{laGPR}) package did not yield good results on the airline delay dataset (RMSE of 24.89 as opposed to 8.75 with the proposed method). The running time for laGP was also considerably longer. Scoring the test set in batches of 15000 rows, the test set prediction took about 50 minutes for the airline delay dataset. The proposed algorithm builds a single model that is used to score the entire test set and completed in about 6.5 minutes.\\
In summary, there are several ways to scale Gaussian Process regression to large datasets. The choice of a particular method should be guided by the characteristics of the problem. The method proposed in this work is appropriate for large datasets that have a small number of important features for which the kernel characteristics are unknown. In such cases exploratory data analysis can be used to determine appropriate kernel types. Additive models may work well for such datasets. Preprocessing such as principal component analysis can be used if needed to make features independent (so that we can use additive models). Stochastic Variational Gaussian Processes and Sparse Gaussian Processes are also good candidates for such problems. Kernel hyper-parameters are learned from the data by these methods. Results of the experiments conducted as part of this study show that the proposed method can match or exceed the performance of the Sparse Gaussian Process or the Stochastic Variational Gaussian Process.

\section{Effect of the Parameters}\label{sec:effect_of_parameters}
Selection of algorithm parameters appropriate for a machine learning task is an arduous task for all practitioners. To alleviate this difficulty, we provide guidelines for parameter selection based on detailed experimentation. The proposed algorithm has three parameters:
\begin{enumerate}
\item The dataset size
\item The subset size
\item The number of estimators
\end{enumerate}
Accordingly, three sets of experiments were performed to capture the effect of each of these parameters on the performance of the algorithm. These experiments are described in this section. 

\subsection{Datasets}\label{sec:datasets}
The following datasets were used in this study:
\begin{enumerate}
\item \textbf{Combined Cycle Power Plant}: This dataset was obtained from the UCI Machine Learning repository (\cite{Lichman}). This dataset has 9568 instances. The target variable is the net hourly electrical power output from a power plant. The dataset has four features.
\item \textbf{Ailerons}: This dataset was obtained from the LIAD(Laboratory of Artificial Intelligence and Decision)(\cite{LIACC}). The target variable for this dataset is the control action associated with the control of a F-16 aircraft. The dataset has 40 features and 7154 instances.
\item \textbf{Elevators}: This dataset was obtained from the LIAD repository (\cite{LIACC}). This dataset is also related to the control of a F-16 aircraft. The target for this dataset is the control action variation for the elevators of the aircraft. The dataset has 6 features and 9517 instances.
\item \textbf{California Housing}: This dataset was obtained from the LIAD repository (\cite{LIACC}). The target variable for this dataset is the median house price. The dataset has 8 features and 20460 instances
\item \textbf{Individual Household Electric Power Consumption}: This dataset was obtained from the UCI Machine Learning repository (\cite{Lichman}). It captures the electric power consumption in a single household over a four year period at a one minute sampling rate. For the experiments in this study, the Voltage was treated as the target variable. Seasonality and periodicity are important characteristics of this dataset (identified during exploratory data analysis of this dataset). For this reason, minute and hour attributes were created out of the time stamp attribute. Similarly, day of week and day of month attributes were created out of the date attribute. This dataset has over 2 million instances and 12 features.
\item \textbf{Airline Delay}:This dataset was obtained from the US Department of Transportation's website (\cite{RITA_Delay_Data_Download}). The data represents arrival delays for US domestic flights during January and February of 2016. This dataset had 12 features and over two hundred and fifty thousand instances. Departure delay is included as one of the predictors while \cite{hensman2013gaussian} does not include it. Also the raw data includes a significant amount of early arrivals (negative delays). For all regression methods considered in this study, better models were obtained by limiting the data to the delayed flights only (arrival delays were greater than zero). This suggests that arrival delays and early arrivals are better modeled separately.

\item \textbf{The Sinc Function}: This is a one dimensional synthetic dataset where the response variable is the sine cardinal function otherwise called the sinc function (noise free). The sinc function is a complex function to learn and is therefore a candidate for this as well as many other machine learning research studies. The dataset had one hundred thousand instances.
\end{enumerate}

\subsection{Effect of Dataset Size}\label{sec:effect_dataset_size}
These experiments study the effect of the size of the dataset ($N$) on the subset size ($\delta$). For each dataset, we pick a fraction of the data elements and determine the subset size ($N^\delta$) required to achieve a target accuracy. The target accuracy is an input parameter for the experiment. We repeat this procedure for various settings of the fraction of the dataset selected ($0.1$ through $1$). The number of estimators for these experiments was maintained at 30. The rationale for this choice is provided in section \ref{sec:effect_num_est}. The results are shown in Figure \ref{fig:DEDVSN} through \ref{fig:SFDVSN}.
\begin{figure}[ht]
\begin{multicols}{2}
\includegraphics[width=\linewidth]{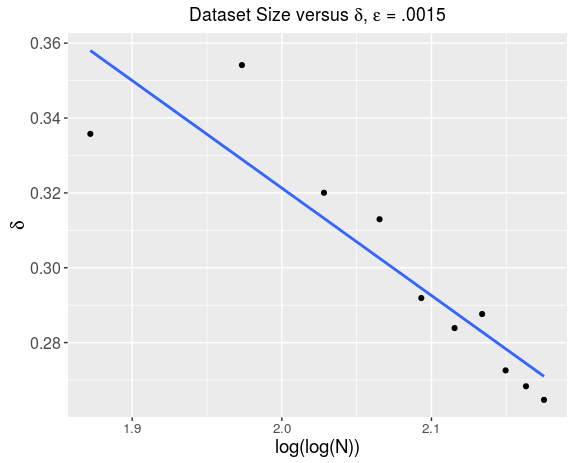}\par\caption{\tiny{Delta Elevators}}
\label{fig:DEDVSN}
\includegraphics[width= \linewidth]{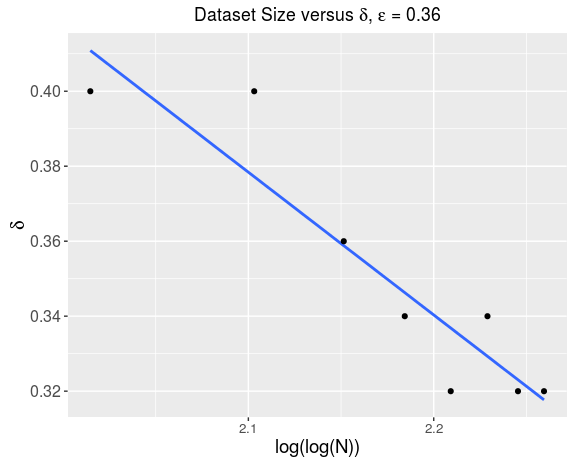}\par\caption{\tiny{California Housing}}
\label{fig:CHDVSN}

\end{multicols}
\begin{multicols}{2}

\includegraphics[width= \linewidth]{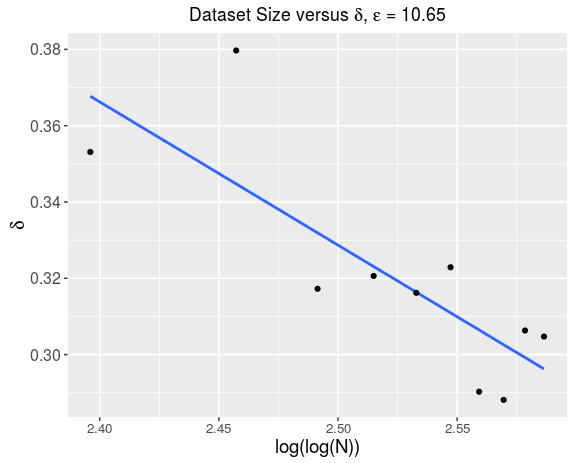}\par\caption{\tiny{Airline Delay}}
\label{fig:ADDVSN}
\includegraphics[width= \linewidth]{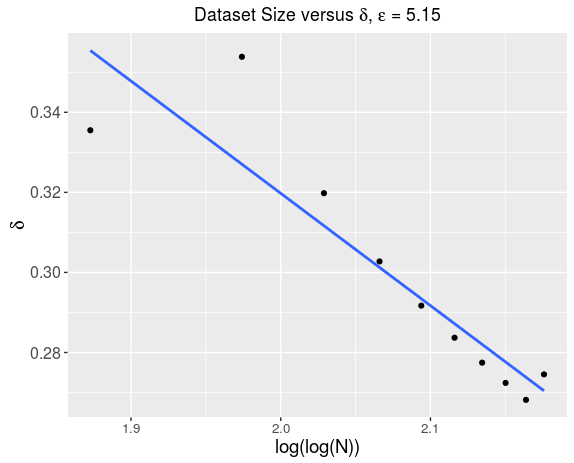}\par\caption{\tiny{Combined Cycle Power Plant}}
\label{fig:PPDVSN}

\end{multicols}
\end{figure}

\begin{figure}[ht]

\begin{multicols}{2}

\includegraphics[width= \linewidth]{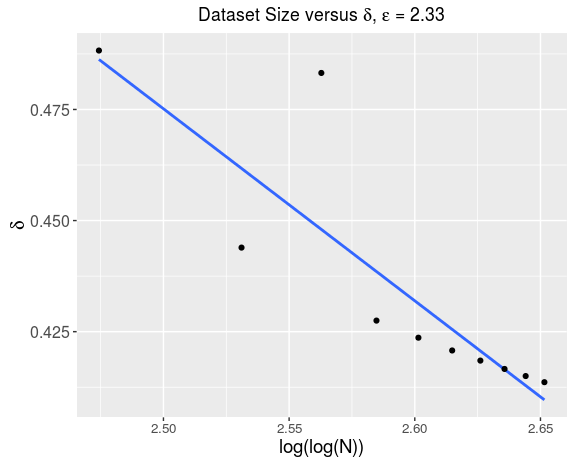}\par\caption{\tiny{House Hold Power Consumption}}
\label{fig:HPDVSN}
\includegraphics[width= \linewidth]{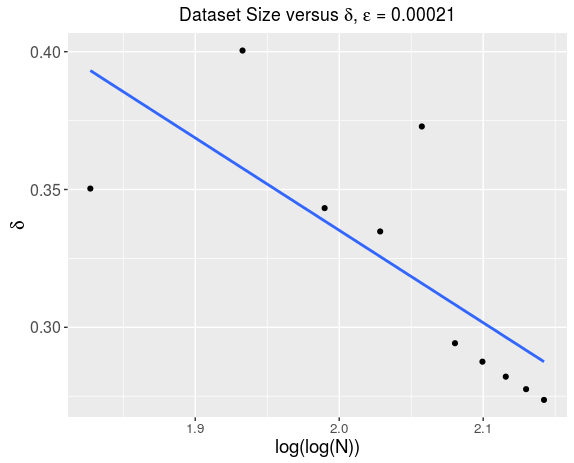}\par\caption{\tiny{Ailerons}}
\label{fig:ALDVSN}
\end{multicols}
\end{figure}

\begin{figure}[H]

\includegraphics[scale = 0.4]{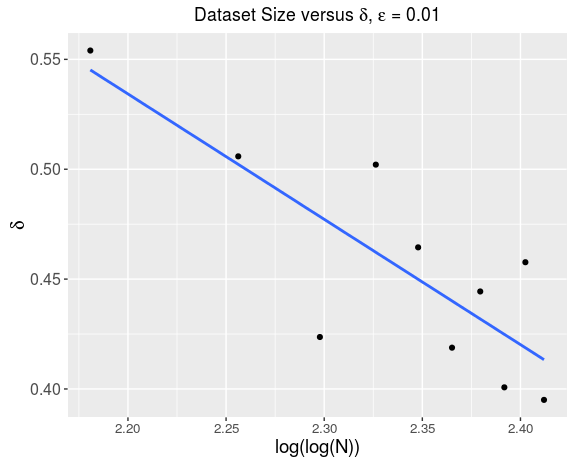}\caption{\tiny{Sinc Function (Synthetic Data)}}
\label{fig:SFDVSN}

\end{figure}
The key observation from these experiments was that the $\delta$ required to maintain a preset accuracy decreases very slowly as $N$ increases. This set of experiments was used to identify candidate choices for the parameters of Equation \ref{eqn:sasanka_estimator}. Since we wanted a function $\delta(N)$, that decreases slowly as $N$ increases, we considered $\{\frac{1}{\sqrt{N}}, \frac{1}{\log(N)}, \frac{1}{\log(\log(N))}\}$. These are slowly decreasing functions of $N$ in decreasing order of slowness. After a rigorous empirical study we found that if we choose $\delta(N) = \frac{1}{\
\log(\log(N))}$ then it works well on all real world datasets and synthetic data.  As discussed in section \ref{ps}, the time complexity of the proposed algorithm is $O(K.N_{s}^{3})$ or $O(K.N^{3.\delta(N)})$. The exponent of $N$ decreases monotonically as $N$ increases. Since $N$ can be expressed as $2^{2^{x}}$, for $N > 2^{2^x}$, the running time is  $ < O(K.N^{\frac{3}{x}})$. The number of estimators $K$, determined empirically, is a constant (about $30$) for all experiments. The rationale for this choice is explained in section \ref{sec:effect_num_est}. The time complexity of the GP computation is therefore  $O(K.N^{\frac{3}{x}})$. So when N is large enough (such that $ x > 3$), the GP computation is sub-linear. For example, when $N = 2^{2^4}$, the time complexity is  $O(K.N^{\frac{3}{4}})$. $g(\epsilon)$ is a monotonically increasing function that characterizes the fact that sample size should increase as the acceptable error threshold 
$\epsilon$ decreases (e.g. $ g(\epsilon) = {\epsilon, \sqrt{\epsilon}, \epsilon^{\frac{1}{10}},\ldots}$). An optimal choice of $g(\epsilon)$ is an area of 
future work. For the experiments reported in this work, $g(\epsilon) = C\epsilon^{\frac{1}{10}}$ with $C = 1$ for low noise datasets ($RMSE << 1$) and $C=0.5$ for noisy datasets ($RMSE > 1$), worked well.

\begin{figure}[H]
\begin{multicols}{2}
\includegraphics[width=\linewidth]{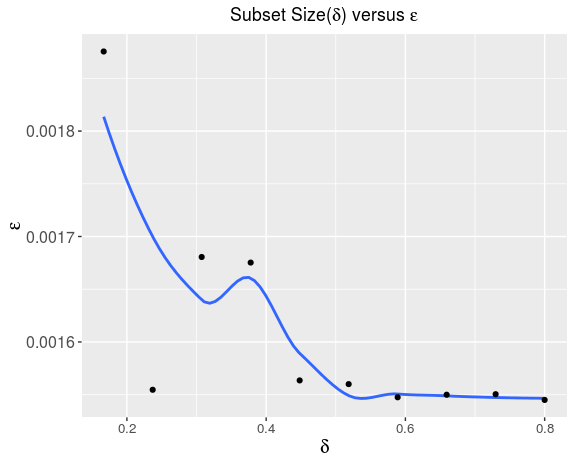}\par\caption{\tiny{Delta Elevators}}
\label{fig:DEEVSD}
\includegraphics[width=\linewidth]{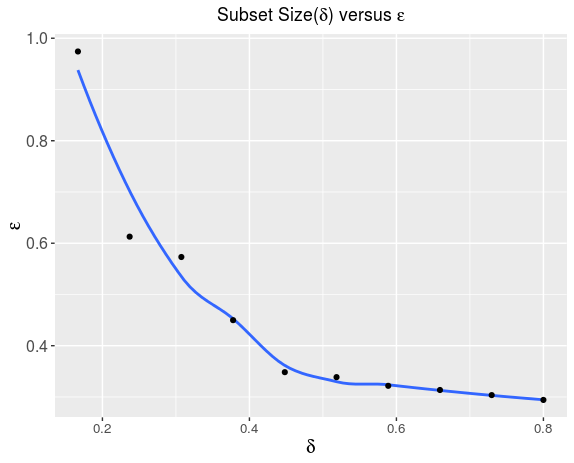}\par\caption{\tiny{California Housing}}
\label{fig:CHEVSD}

\end{multicols}
\end{figure}
\begin{figure}[H]
\begin{multicols}{2}
\includegraphics[width=\linewidth]{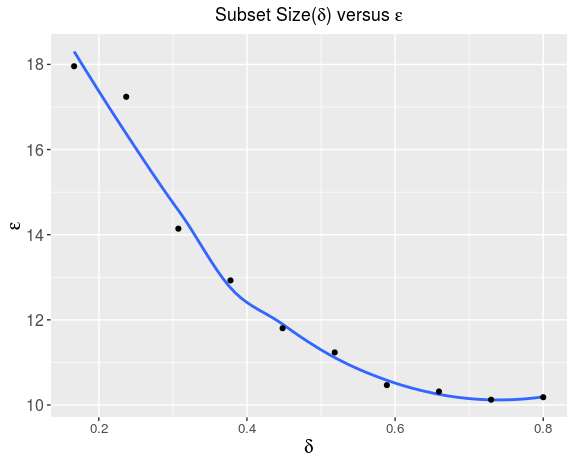}\par\caption{\tiny{Airline Delay}}
\label{fig:ADEVSD}
\includegraphics[width=\linewidth]{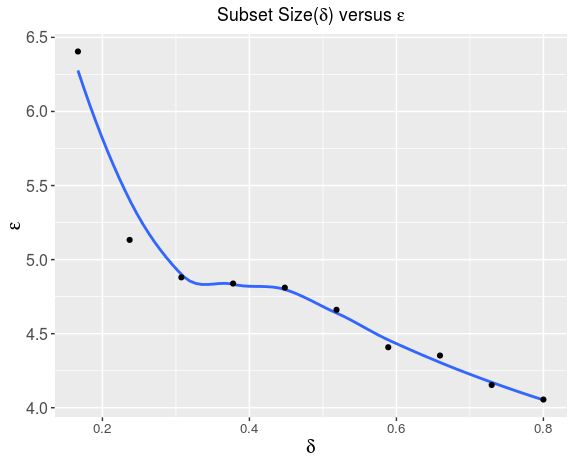}\par\caption{\tiny{Combined Cycle Power Plant}}
\label{fig:PPEVSD}

\end{multicols}
\end{figure}

\begin{figure}[H]
\begin{multicols}{2}

\includegraphics[width=\linewidth]{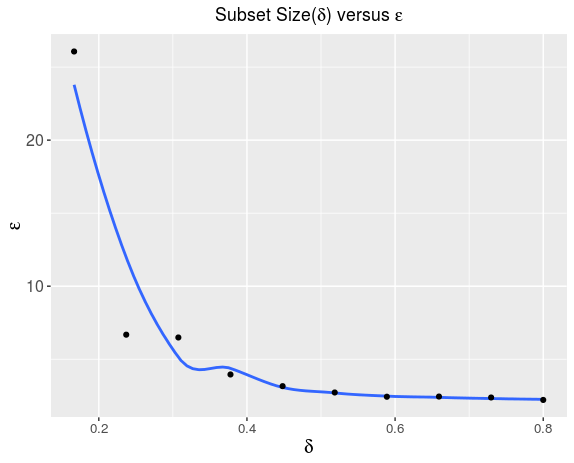}\par\caption{\tiny{House Hold Power Consumption}}
\label{fig:HPEVSD}
\includegraphics[width=\linewidth]{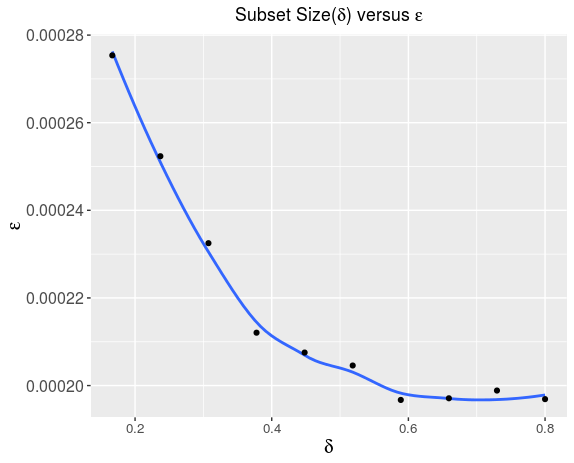}\par\caption{\tiny{Ailerons}}
\label{fig:ALEVSD}
\end{multicols}
\end{figure}

\begin{wrapfigure}{r}{0.4\textwidth}
\includegraphics[scale = 0.4]{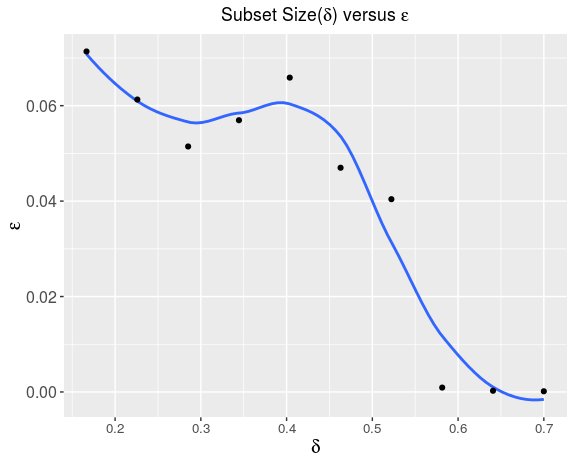}\caption{\tiny{Sinc Function (Synthetic Data)}}
\label{fig:SFEVSD}
\end{wrapfigure}

\subsection{Effect of Subset Size}\label{sec:effect_subset_size}
These experiments explore the effect of the subset size, captured by the parameter $\delta$,  on the accuracy ($\epsilon$), for a fixed dataset size ($N$). For each dataset, a fraction of the dataset is picked for this experiment. This represents the data for the experiment ($N$). We pick $N$ such that $N^{1.0}$ (i.e, $\delta = 1.0$) is computationally tractable (about 2000).  The subset $N_s$, used for Gaussian Process model development was $N^\delta$. We fix the number of estimators to be 30 (see section \ref{sec:effect_num_est} for rationale). For each $\delta$ value in a range of values, we record the RMSE ($\epsilon$). The key insight from these experiments was that complex functions, like the Sinc function (see Figure \ref{fig:SFEVSD}) needed a larger subset size to produce a given level of accuracy. Note that we can expect the $\delta$ to drop with the increase in dataset size because of the effect reported in section \ref{sec:effect_dataset_size}.  Section \ref{sec:accuracy} provides the $\delta$ values associated with the full dataset. The results for these experiments are shown in Figure \ref{fig:DEEVSD} through \ref{fig:SFEVSD}.

\subsection{Effect of Number of Estimators}\label{sec:effect_num_est}

These experiments capture the effect of the number of estimators ($K$) on the accuracy of the algorithm ($\epsilon$), for a particular set of $\delta$ and $N$ values.

\begin{figure}[H]
\begin{multicols}{2}
\includegraphics[width=\linewidth]{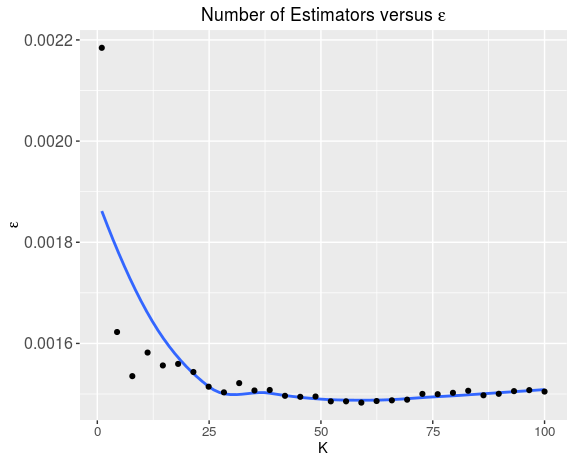}\par\caption{\tiny{Delta Elevators}}
\label{fig:DEKVSE}
\includegraphics[width=\linewidth]{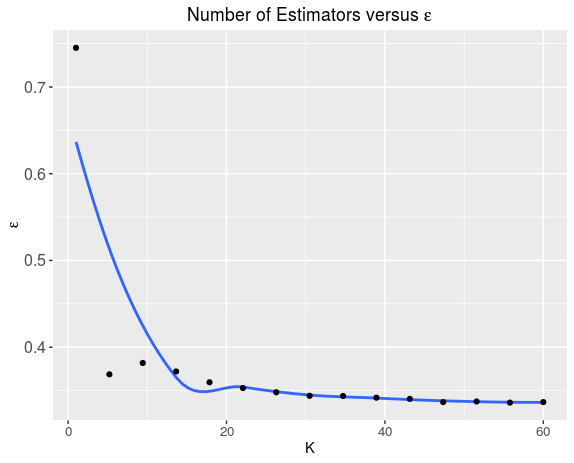}\par\caption{\tiny{California Housing}}
\label{fig:CHKVSE}
\end{multicols}
\end{figure}
\begin{figure}[H]
\begin{multicols}{2}
\includegraphics[width=\linewidth]{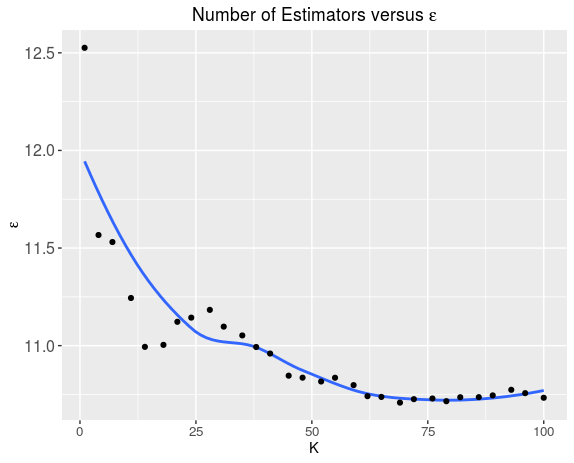}\par\caption{\tiny{Airline Delay}}
\label{fig:ADKVSE}
\includegraphics[width=\linewidth]{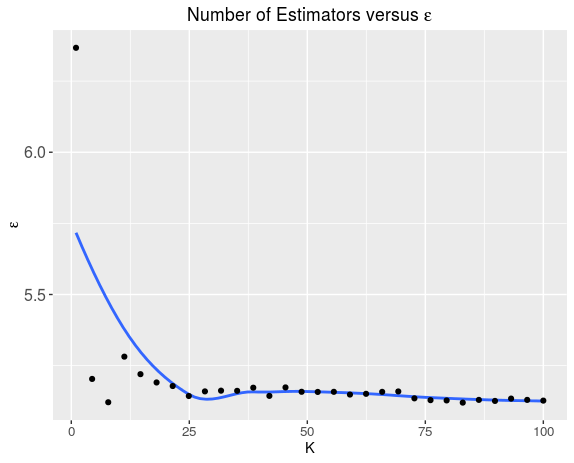}\par\caption{\tiny{Combined Cycle Power Plant}}
\label{fig:PPKVSE}
\end{multicols}
\end{figure}
\begin{figure}[H]
\begin{multicols}{2}
\includegraphics[width=\linewidth]{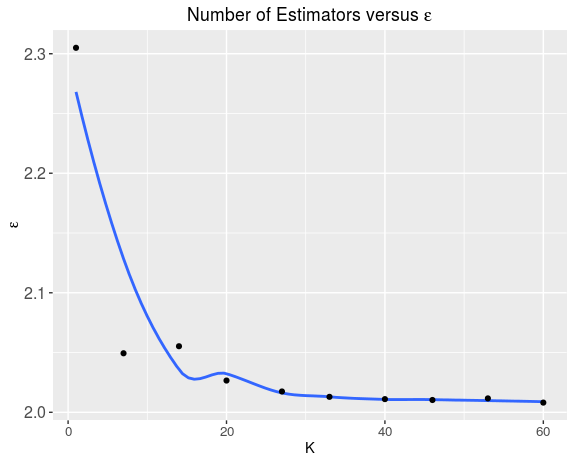}\par\caption{\tiny{House Hold Power Consumption}}
\label{fig:HPKVSE}
\includegraphics[width=\linewidth]{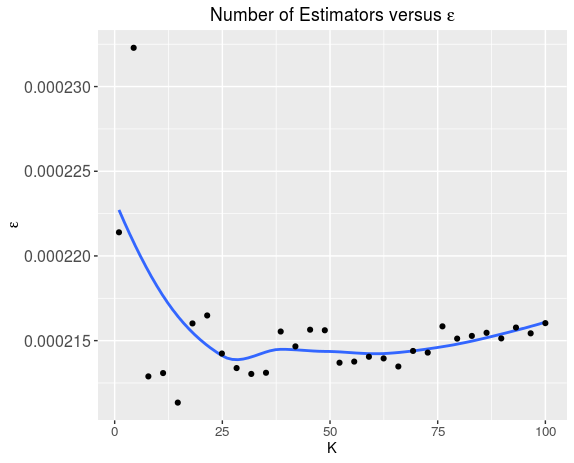}\par\caption{\tiny{Ailerons}}
\label{fig:ALKVSE}
\end{multicols}
\end{figure}
\begin{wrapfigure}{r}{0.4\textwidth}
\includegraphics[scale = 0.4]{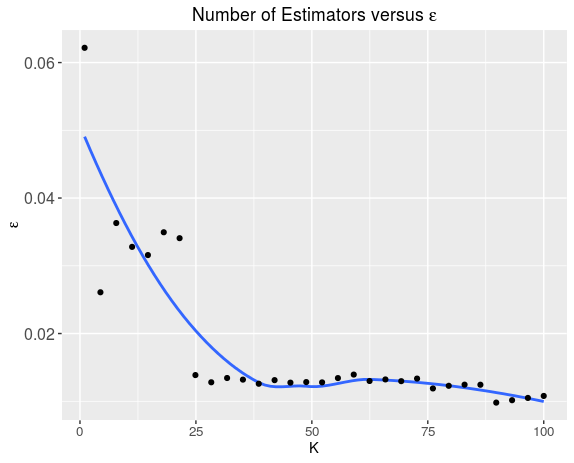}\caption{\tiny{Sinc Function (Synthetic Data)}}
\label{fig:SFKVSE}
\end{wrapfigure}

For each dataset the following experiment was performed. The fraction of the dataset to use for the experiment ($N$) and the subset size ($\delta$) are selected. For each $K$ in a range of values, Algorithm \ref{algo_resamp} is applied and the error ($\epsilon$) is recorded. The key insight from this set of experiments was that given a subset size($N_s$), there is a point upto which increasing the number of estimators ($K$) drops the error, but once a threshold value is reached (around $30$, for all the datasets), there is little benefit in increasing the number of estimators. A plausible explanation for this behavior could be that increasing the number of estimators reduces the variance component of the error in the bias-variance decomposition of the error($\epsilon$). The results for these experiments are shown in Figure \ref{fig:DEKVSE} through \ref{fig:SFKVSE}.

\section{Application of the Algorithm} \label{ards}
In this section we describe the results of applying the algorithm reported in this work to the datasets described in section \ref{sec:datasets}.
\subsection{Independent Performance Assessments} \label{sec:independent_assements}
It may of interest to see how the estimates from Gaussian Progress regression compare to estimates from other methods for large regression tasks. We report the performance of two methods. XGBoost (\cite{chen2016xgboost}) is a recent tree based algorithm using the gradient boosting framework that is very scalable and gaining adoption among practitioners. Trees partition the predictor space and can account for interaction. This method uses boosting instead of bagging. Difference in accuracy estimates between XGBoost and the proposed method for a particular dataset could be attributed to the influence of boosting or effects of interaction among variables. In most cases the estimates from XGBoost were comparable to the estimates from the proposed method. As discussed in \ref{sec:stacking}, it is possible to combine estimates from XGBoost with the estimates from the proposed method using stacking. The Generalized Additive Model (GAM) (\cite{friedman2001elements}) is a scalable regression technique. As the name suggests, it fits an additive model in terms of smooth non-parametric functions (typically splines) of the predictor variables. The GAM estimate serves as an independent performance estimate from another non-parametric regression algorithm based on an additive model. In most cases the accuracy obtained from GAM's were similar to those obtained from Gaussian Process regression.
\subsection{Feature Relavance}
As discussed in section \ref{ps}, an additive model worked well for the datasets used in this study. Further, in all these datasets, the response depended only a small number of attributes. XGBoost can report feature importance. This summarized in Table \ref{tab:feature_imp} below. The response depends on a small subset of predictors in all the datasets reported in this study. GPy, (\cite{gpy2014}) the package used to implement the experiments reported in this work, implements the Automatic Relevance Determination (see \cite[Chapter 5]{Rasmussen06gaussianprocesses}) feature. Features that are not relevant are dropped from the model.   

\begin{wraptable}{r}{ 8cm}
\begin{savenotes}
\scriptsize
\begin{tabular}{|r|l|c|c|}
  \hline
 & Dataset & Features & Important Features \\ 
  \hline
1 & Airline Delay & 11 & 1 \\ 
  2 & Ailerons & 40 & 3 \\ 
  3 & Power Plant & 4 & 2 \\ 
  4 & Delta Elevators & 6 & 2 \\ 
  5 & California Housing & 8 & 3 \\ 
  6 & House Hold Power & 12 & 4 \footnotemark \\ 
   \hline
\end{tabular}
\caption{Feature Importance}

\label{tab:feature_imp}
\end{savenotes}
\end{wraptable}
\footnotetext{timestamp is counted as a feature. Variables derived from timestamp are not included in this count.}
\normalsize

\subsection{Accuracy}\label{sec:accuracy}
As discussed in section \ref{pw}, the choice of a GP method to use for a regression task depends on the problem context. The algorithm presented in this work does not require the hyper-parameter values to be specified. In terms of the inputs to the algorithm, the algorithm presented in this work is similar to the Sparse Gaussian Process and the Stochastic Variational Gaussian Process. These algorithms only require the dataset, kernel and the subset size as input. We report the performance of the Sparse Gaussian Process, Stochastic Variational Gaussian Process and the proposed algorithm for each dataset. These algorithms are good choices to begin the knowledge discovery process in large datasets. The results of applying Algorithm \ref{algo_resamp} to all the datasets listed in section \ref{sec:datasets} are shown in Table  \ref{tab:perf_accuracy}.
\begin{table}[ht] 
\scriptsize

\begin{tabular}{|l|r|r|r|r|r|r|r|r|}
  \hline
 Dataset & BM1 & BM2 & POE & SVGP & SPGP  & XGBoost & GAM & SD \\ 
  \hline
   Ailerons & 0.000214 & 0.000208 & 0.000220 & 0.000220 & 0.000220 & 0.000220 & 0.000200 & 0.000410 \\ 
   Delta Elevators & 0.001510 & 0.001450 & 0.001547 & 0.001460 & 0.001460 & 0.001432 & 0.001444 & 0.002370 \\ 
  CCPP & 4.32 & 4.24 & 4.27 & 5.19 & 4.10 & 3.73 & 4.11 & 17.07 \\ 
  Cal Housing & 0.293 & 0.294 & 0.293 & NA & NA & 0.240 & 0.281 & 0.569 \\ 
  Airline Delay & 8.75 & 8.75 & 8.74 & 8.85 & 10.94 & 8.74 & 9.45 & 31.31 \\ 
  HPC & 2.039 & 1.992 & 2.120 & NA & NA  & 1.63 & 2.18 & 3.24 \\ 
  Sinc Function & 0.0267 & 0.0200 & 0.0132 & 0.0627 & 0.0170 & NA  & NA & 0.0634 \\ 
   \hline
\end{tabular}
\caption{Performance Accuracy}
\label{tab:perf_accuracy}

\end{table}
\normalsize
For each dataset, Table \ref{tab:perf_accuracy} includes the following columns:
\begin{enumerate}
\item BM1 (Bagging Method 1): This is the result of applying Algorithm \ref{algo_resamp} using the method specified in section \ref{sec:emp_exp_est} to select the subset.
\item BM2 (Bagging Method 2): This is the result of applying Algorithm \ref{algo_resamp} using the method specified in section \ref{sec:emp_est_sasanka} to select the subset.
\item POE: This is the result of applying the Product of Experts Algorithm (Equation \ref{eqn:poe}) instead of bagging. The subset size used is the same as that for Bagging (method 1).
\item SVGP: This is the result of applying the SVIGP algorithm. 
\item SPGP: This is the result of applying the Sparse GP algorithm. 
\item XGBoost: This is the result obtained from the XGBoost algorithm.
\item GAM: This is the result obtained from GAM.
\item SD: This is the standard deviation of the response. If we use a model that simply predicts the mean response for each test point, then the standard deviation represents the error associated with such a model. For regression models to be useful, they must perform better than this model.
\end{enumerate}

For all datasets the split between the training and the test sets was 70\% - 30\%. The reported accuracies are on the test set. The subset sizes required for each dataset is captured by the $\delta$ parameter. The $\delta$ values associated with Table \ref{tab:perf_accuracy} are presented in Table \ref{tab:dataset_deltas}. The $\delta$ values for the bagging methods and POE correspond the most accurate estimates we could obtain for the datasets on a laptop with 16 GB of RAM. The SVGP and the SPGP columns of Table \ref{tab:dataset_deltas}, represent the $\delta$ at which the best performance was obtained for the Stochastic Variational GP and the Sparse GP methods. For the airline dataset, the best performance for Stochastic Variational GP was obtained at a smaller subset size than the one used for the proposed method. For the airline dataset, Sparse GP showed no improvement in performance for $\delta$ greater than $0.33$

\begin{wraptable}{r} { 8cm}
\centering
\small
\begin{tabular}{|r|l|r|r|r|r|r|}
\hline
 & Dataset & BM1 & BM2 & POE & SPGP & SVGP\\ 
  \hline
 & Ailerons & 0.30 & 0.30 & 0.30  & 0.30 & 0.30\\ 
  2 & Delta Elevators & 0.30 & 0.30 & 0.30 & 0.30 & 0.30 \\ 
  3 & CCPP & 0.6 & 0.6 & 0.6 & 0.6 & 0.6 \\ 
  4 & Cal Housing & 0.625 & 0.625 & 0.625 & NA & NA \\ 
  5 & Airline Delay & 0.56 & 0.49 & 0.0.56 & 0.33 & 0.523\\ 
  6 & HPC & 0.43 & 0.38 & 0.43 & NA & NA \\ 
  7 & Sinc Function & 0.50 & 0.50 & 0.50 & NA & NA \\ 
   \hline
\end{tabular}
\caption{$\delta$ Requirements for the Datasets}
\label{tab:dataset_deltas}
\end{wraptable}
\normalsize
The kernels used for the datasets are shown in Table \ref{tab:dataset_kernels}.
\begin{wraptable}{r} { 7cm}
\centering
\scriptsize
\begin{tabular}{|r|p{2 cm}|p{3 cm}|}
  \hline
 & Dataset & Kernel \\ 
  \hline
1 & Ailerons & sum of linear and RBF \\ 
  2 & Delta Elevators & sum of linear and RBF \\ 
  3 & CCPP & sum of linear and RBF \\ 
  4 & California Housing & sum of PeriodicMatern32, Linear, RBF and a product kernel of Linear and RBF on the medianIncome attribute \\ 
  5 & Airline Delay & sum of BF, linear and White Noise \\ 
  6 & HPC & sum of Bias, Cosine, RBF, Linear and Brownian Motion \\ 
  7 & Sinc Function & RBF \\ 
   \hline
\end{tabular}
\caption{$\delta$ Kernels for the Datasets}
\label{tab:dataset_kernels}

\end{wraptable}
\normalsize
The experiments reported in this work used \cite{gpy2014} for implementation. The Sparse GP implementation in this package is based on the Variational Sparse GP method (see \cite{titsias2009variational}). The Stochastic Variational GP implementation is the one used in \cite{hensman2013gaussian}. Both these methods use Stochastic Gradient Descent for hyper-parameter learning. Kernel support for these methods is also limited to simple kernels and kernels like the Brownian Motion Kernel or the Periodic Matern kernel are not supported. For this reason, we do not report accuracy for SVIGP or the Sparse GP for datasets that required these complex kernels (Household Power Consumption and California Housing ). This also highlights the fact that Algorithm \ref{algo_resamp} may be a good candidate for datasets with the desired characteristics described earlier but requiring a complex kernel to model the underlying function.

An analysis of Table \ref{tab:perf_accuracy} shows that the proposed method performs as well as XGBoost or GAM's in most cases. In some cases, XGBoost does marginally better. In these cases we explored if it is possible to construct a better estimator using both these estimators. This is discussed next.

\subsection{Combining Estimators}\label{sec:stacking}
Combining estimators is not a new idea \cite{wolpert1992stacked}. More recently, \cite{lloyd2014gefcom2012} has combined gradient boosted tree models with Gaussian Process models. This prompted us to explore combining estimators for the data sets where the Gaussian Process model produced a slightly lower accuracy. There were three datasets where the performance of the Gaussian Process model was slightly lower than the XGBoost model. These were the Combined Cycle Power Plant, California Housing and the Household Power Consumption dataset. To combine estimators, the output of the XGBoost models and the GP models were used as the inputs to a classifier (the stacking classifier). The response from the classifier was the best model for a given set of XGBoost and GP model responses. A K-nearest neighbor classifier was used for this purpose. The best value of K was determined through experimentation. Given a test point, the estimates from the XGBoost model and the GP model can be obtained. The classifier then predicts best model for this set of estimates. This model is then used to provide the estimate for the test point. This procedure improved the accuracy for the California Housing and Combined Cycle Power Plant datasets. A K-nearest neighbor regression performed better than the K-nearest neighbor classifier for the Household Power Consumption dataset. The results are shown in Table  \ref{tab:stacked_estimates}
\begin{wraptable}{r} { 7cm}
\centering
\scriptsize

\begin{tabular}{|l|l|c|}
  \hline
 & Dataset & RMSE  \\ 
  \hline
 1 & California Housing & 0.211 \\ 
 2 & CCPP & 2.943 \\ 
 3 & HPC & 1.610 \\ 
   \hline
\end{tabular}
\caption{$\delta$ Accuracy using XGBoost and bagged GP models}
\label{tab:stacked_estimates}

\end{wraptable}
\normalsize

A comparison of Table \ref{tab:stacked_estimates} with Table \ref{tab:perf_accuracy} shows that combining estimators yields solutions that are more accurate. The idea of combining estimators can be refined in many ways. We have not included the Sparse GP and Stochastic Variational GP in this solution. The choice of the classifier or regression model to combine estimates from the component models is another modeling decision. The intent here is to show that it is possible to combine the GP model developed using the method presented in this work with other regression models to produce solutions that are better than the individual solutions. An optimal choice of the estimators and the method used for stacking is beyond the scope of this work.

\section{Conclusion} \label{dr}
There are many methods to scale Gaussian Process regression to large datasets. The appropriate choice depends on the problem context. The proposed method is appropriate for large datasets with a small set of important features for which the kernel characteristics are unknown. Kernel choices can be determined through exploratory data analysis \cite{duvenaud2014automatic}. Kernel hyper-parameters can be learned from the data. The data for the experiments reported in this work came from diverse application areas and in a wide range of sizes (few thousand to two million). In these datasets, the target variable depended on a small set of variables and an additive model matched the performance of models that permit interaction like XGBoost. This suggests that datasets with these characteristics are not uncommon. Results from Minimax theory for non-parametric regression indicate that additive models that depend on a small set of predictors have an attractive rate of convergence. This suggests that we can develop adequate regression models with a small subset of samples from the dataset. This was consistent with results observed in the experiments conducted as part of this study. The Stochastic Variational Gaussian Process and the Sparse Gaussian Process are also good candidates for problems with these characteristics.  The results of this study show that the proposed algorithm can match or exceed the performance of the Sparse Gaussian Process or the Stochastic Variational Gaussian Process. The results of this study also show that Gaussian Processes can be as effective as ensemble methods like Gradient Boosted Trees on large datasets.  Gaussian Processes are based on a probabilistic frame work and can provide uncertainty estimates directly as compared to other tools for large regression tasks like XGBoost. This could very important for some applications. For example, an analyst may be interested in the probability of a particular amount of delay given information for a particular flight. The regression function is also interpretable in the case of Gaussian Process models in contrast to methods like gradient boosted trees. Therefore Gaussian Process models can be good explanatory models as well as good predictive models. An important feature of this algorithm is the simplicity of implementation. In Internet applications that process continuous streams of data, frequent model development and deployment is needed. An algorithm that is simple but effective may fit these applications well. Finally it should be noted that it is possible to combine this algorithm with other algorithms like Gradient Boosted Trees using model stacking to achieve performance gains.
\section{References}
\bibliography{biggp}

\end{document}